\documentclass[letterpaper, 10 pt, conference]{ieeeconf}  

\IEEEoverridecommandlockouts                              

\usepackage{color}
\usepackage{multicol}
\usepackage[bookmarks=true]{hyperref}

\usepackage{gensymb}
\usepackage[pdftex]{graphicx}
\usepackage[cmex10]{amsmath}
\usepackage[ruled,vlined]{algorithm2e}
\usepackage{array}
\usepackage[caption=false,font=footnotesize]{subfig}
\usepackage{url}

\usepackage{amsthm}
\usepackage{amssymb}
\usepackage{listings}
\usepackage{tikz}
\usetikzlibrary{automata,positioning,shapes,arrows}

\graphicspath{{./pics/}}
\DeclareGraphicsExtensions{.pdf,.jpeg,.png}

\theoremstyle{definition}
\newtheorem{defn}{Definition}
\newtheorem{problem}{Problem}
\newtheorem{example}{Example}
\newtheorem{prop}{Proposition}

\title{\LARGE \bf SafeGuardPF: Safety Guaranteed Reactive Potential Fields for Mobile Robots in Unknown and Dynamic Environments*
}

\author{Rafael Rodrigues da Silva$^{1,2}$, Samuel Silva$^{1,2}$, Grigoriy Dubrovskiy$^{1}$ and Hai Lin$^{1}$
\thanks{*The financial supports from NSF-CNS-1239222, NSF-CNS-1446288  and NSF- EECS-1253488 for this work are greatly acknowledged.}
	\thanks{$^{1}$ All authors are with Department of Electrical Engineering, University of Notre Dame, Notre Dame, IN 46556, USA.{\tt\small (rrodri17@nd.edu;~ssilva1@nd.edu;} {\tt\small ~gdubrovs@nd.edu;~hlin1@nd.edu)}}
\thanks{$^{2}$ The first and second authors would like to appreciate the scholarship support by CAPES/BR, BEX 13242/13-0}
}

\begin{document}

\maketitle
\thispagestyle{empty}
\pagestyle{empty}

\begin{abstract}
An autonomous navigation with proven collision avoidance in unknown and dynamic environments is still a challenge, particularly when there are moving obstacles. A popular approach to collision avoidance in the face of moving obstacles is based on model predictive algorithms, which, however, may be computationally expensive. Hence, we adopt a reactive potential field approach here. At every cycle, the proposed approach requires only current robot states relative to the closest obstacle point to find the potential field in the current position; thus, it is more computationally efficient and more suitable to scale up for multiple agent scenarios. Our main contribution here is to write the reactive potential field based motion controller as a hybrid automaton, and then formally verify its safety using differential dynamic logic. In particular, we can guarantee a passive safety property, which means that collisions cannot occur if the robot is to blame, namely a collision can occur only if the robot is at rest. The proposed controller and verification results are demonstrated via simulations and implementation on a Pioneer P3-AT robot.
\end{abstract}

\section{Introduction}

The autonomous navigation of mobile robots in unknown environments is a complex task that still has much room for improvement. To be autonomous, a robot must be able to interpret what is being read of the environment and associate with its previous knowledge of the place to take actions towards the goal in an effective manner \cite{hamani2012mobile}. Furthermore, the robotic motion planning in unknown and dynamic environments can be applied to several areas, such as: intelligent vehicles on highways, air and sea traffic control, automated assembly and animation \cite{shiller2001motion}. 

Existing motion planning approaches can be roughly divided into two categories: global path planning and local path planning \cite{raja2012optimal}. A global path planning, such as symbolic motion planning\cite{lin2014mission}, is usually based on a discretization of the working-space with an assumption of a static and known environments. On the other hand, a local path planner relies on local sensing and reaction to avoid collisions in possibly unknown environments. Typical examples include the potential field based navigation \cite{ge2000new}, dynamic window approaches \cite{thrun1997dynamic}, and sampling-based approaches \cite{raja2012optimal}. 



Although these existing local path planning methods are widely applied and effective to handle uncertain and dynamic environments, they often lack formal safety guarantees. However, in many safety critical applications, such as the neutralization of mass destruction weapons or clean-up in a nuclear power plant meltdown, guaranteed safety cannot be overemphasized. Motivated by this gap, recent years have seen an increase of research activities on introducing formal methods into robotic motion planning. For example, the PASSVOID \cite{bouraine2012provably} computes an online verification algorithm which searches for safe controls that lead a mobile robot to avoid braking inevitable collision states (ICS). A braking ICS is a variant of the concept of ICS (inevitable collision states) \cite{fraichard2004inevitable} in which are considered the braking trajectories and the collision will not occur in the future when the vehicle is not at rest. The approach proposed by \cite{althoff2012safety}, \cite{althoff2014online} and \cite{hess2014formal} require the computation of a set of reachable trajectories. In \cite{althoff2012safety}, the authors proposed a probabilistic approach to rank these trajectories. In \cite{althoff2014online} and \cite{hess2014formal}, the authors considered nonlinear controllers and bounded disturbances. These methods require searching for safe trajectories online, and hence can be computationally expensive.

Differently, in \cite{mitsch2013provably} and \cite{mitsch2016formal}, the authors presented an offline verification in differential dynamic logic d$\mathcal{L}$ \cite{platzer2010logical} to guarantee a passive or passive friendly safety of motion planning for ground vehicles. A passive or passive friendly safety property is a specification for motion planning which considers that ICS-free maneuvers cannot be guaranteed because of the limited sensors range and elusive nature of the future \cite{macek2009towards}. Thus, the proposed verification considered two safety properties: passive safety, in which it was assumed to know only the maximum velocity of the moving obstacles; and passive friendly safety, in which was also assumed to know the obstacle braking power and delay to start braking. Although the offline verification allows avoiding online computation of the safe trajectories, it was applied for a Dynamic Window Approach (DWA) algorithm, which requires finding an optimized trajectory towards a goal position, and it can also be computationally costly. 

In this paper, we propose to adopt the reactive potential field approach for mobile robots. 
Artificial Potential Fields were first proposed in \cite{khatib1986real}, where the obstacles were interpreted as repulsion fields and the goal as an attraction field, and the path was generated by following the optimal trajectory along the intersection of all fields in the environment. Artificial Potential Fields have been widely adopted in robotic motion planning, see e.g. \cite{sun2016flight}, \cite{budhraja2016approach}, \cite{matoui2015local} and the references therein.  Our main goal is to formally prove that motion planning based on artificial potential fields is passive safe. Technically, we first model the controller based on potential fields as a hybrid automaton, and then 
use differential dynamic logic to formally verify the safety of controlled robotic motion dynamics. Its passive safety is achieved by leveraging results from [14,15] that give a constraint to the state variable that ensures the passive safety. In fact, we find the maximum safe velocity with this constraint. Thus, the potential fields are designed such that the robot heading points towards maximize this velocity while pointing towards the destination. The proposed  approach was implemented on a Pioneer P3-AT robot, and simulation and experimental results are presented. Since this approach guarantees a safety property for moving obstacles, we call it SafeGuardPF (Safety Guaranteed reactive Potential Field). 

In summary, the contribution of this work is to propose a safety guaranteed reactive potential field for moving obstacles. This potential field approach is formally proven to ensure a passive safety while being computationally efficient which, to the best of our knowledge, has not been attempted before:
\begin{itemize}
	\item Unlike \cite{buniyamin2011simple}, \cite{vseda2007roadmap} and \cite{thrun1997dynamic}, where the motion plan is formally verified for static obstacles, in our approach, the safety is ensured for moving obstacles. 
    \item Unlike \cite{vadakkepat2000evolutionary} and \cite{yin2009new}, where potential field approaches were presented without safety verification, in our approach, the safety is formally proved.
    \item Unlike \cite{ge2002dynamic}, where a potential field approach for moving obstacles is proposed, the SafeGuardPF does not require full knowledge of those obstacles dynamics but only the maximum velocity. Further, the approach presented in this work is reactive, meaning that it takes into account the current sensor readings every control cycle. Thus, it is robust to a partially unknown and dynamic environment.
    \item Unlike the approaches based on a verification presented in \cite{bouraine2012provably}, \cite{althoff2012safety}, \cite{althoff2014online} and \cite{hess2014formal}, where uses an online verification, in this approach, the state constraints that ensures the safety property is designed and verified offline. Thus, less online computation is required to guarantee the safety.
    \item Unlike \cite{mitsch2013provably} and \cite{mitsch2016formal}, which implement a formally proven DWA approach for moving obstacles, the potential field approach proposed here does not require online numerical optimization algorithms. Thus, it can be more computationally efficient.
\end{itemize}

The rest of the paper is organized as follows. Section \ref{sec:preliminaries} presents background information on potential fields. Section \ref{sec:system} presents the scenario and problem formulation of this work. Section \ref{sec:approach} describes in details the proposed potential field approach. Simulation and experimental results are presented in Section \ref{sec:results}. Finally, Section \ref{sec:conclusion} concludes the paper.

\section{Preliminaries}\label{sec:preliminaries}

\subsection{Potential Fields}

 As originally proposed in \cite{khatib1986real}, the philosophy of the Artificial Potential Fields (APF)  can be schematically described for mobile robot navigation as the agent moving in a field of forces where the position to be reached is an attractive pole and the obstacles are repulsive surfaces for the robot.
    In general, the field is a composition of this two fields
\begin{equation}\label{eq:totalpotential}
\vec{\mathcal{U}}_{tot} = \vec{\mathcal{U}}_{att}+\vec{\mathcal{U}}_{rep}. \\
\end{equation}

Each field function can be defined in order to favor the movement of robots according to their own physical limitations or environmental limitations. Even though there are several modifications from the original potential functions proposed on \cite{khatib1986real}, they all have some common parameters like $K_{rep}$, $K_{att}$ and $p_0$ that define the variation of the field as well as its influence radius. In order to obtain the expression for the virtual force to which the agent is subject along the trajectory, the gradient of each field can be used, as seen in Equations (\ref{eq:attforce}) and (\ref{eq:repforce}) as below:
\begin{align}\label{eq:attforce}
\vec{\mathcal{F}}_{att} = -\nabla\vec{\mathcal{U}}_{att}\\\label{eq:repforce},
\vec{\mathcal{F}}_{rep} = -\nabla\vec{\mathcal{U}}_{rep}.
\end{align}

The attraction field should generate a vector field point towards the target in a manner that the further the agent is from the target, the bigger the attraction force. It was first proposed in \cite{khatib1986real} that the attraction field can be described as in Equation (\ref{eq:attfield}).
\begin{equation}\label{eq:attfield}
\vec{\mathcal{U}}_{att} = 1/2*K_{att}*(\vec{{u}}_{rob}-\vec{{u}}_{goal}), \\
\end{equation}
where $K_{att}$ is the attraction gain, responsible for adjusting the intensity of the field, and the difference of the vectors $\vec{{u}}_{rob}$ and $\vec{{u}}_{goal}$ represents the distance between the robot and its goal.

On the other hand, the repulsive field should generate forces that would push the robot to the opposite direction. In general each obstacle has an radius of influence and the field increases its repulsion intensity as decreases the distance between the obstacle and the robot. One of the most common repulsion field functions is shown below:
\begin{align} \label{eq:repfield}
\vec{\mathcal{U}}_{rep} = & \begin{cases} 
1/2*\eta*(\frac{1}{\rho}-\frac{1}{\rho_0})^2 \text{,} & \text{if } \rho \leq \rho_0. \\
0 \text{,} & \text{otherwise,}
\end{cases} 
\end{align}
where $\eta$ is the repulsive gain, $\rho$ is the shortest distance between the obstacle and the agent, and $\rho_0$ is the limit distance that defines the influence range of the obstacle.

    Some methods have been proposed to obtain the coefficients of these equations, but as mentioned in \cite{savkin2015safe}, some artificial potential field methods have been extended to moving obstacles, though without rigorous justification. So there is no guarantee that they will always generate a safe trajectory.

\section{System Model and Problem Formulation}\label{sec:system}
This work considers a mobile ground robot whose workspace is unknown and dynamic, and the obstacles can be moving up to a known maximum velocity $V$. This robot can realize forward circular trajectories defined by translational $v: v \geq 0$ and angular velocities $\omega$ that are controlled by a regulator. The translational acceleration $a$ and the angular velocity are assumed to be bounded with known bounds (i.e. $-b \leq a \leq A: b > 0 \wedge A \geq 0$ and $-\Omega \leq \omega \leq \Omega : \Omega \geq 0$). This regulator changes its actuators signals cyclically with a non-deterministic period with known maximum value $\epsilon$. Therefore, every control cycle, a pair of translational and angular velocity $\langle v^*, \omega^* \rangle$ defines the mobile robot trajectory to be realized for a short range time $\epsilon$. Several types of vehicles can realize this trajectory, such as differential drive, Ackermann drive, single wheel drive, synchro drive, or omni drive vehicles \cite{braunl2008embedded}. It is also assumed that the robot has a sensor to measure its relative distance to the obstacles in its environment. Since potential field approaches can lead to local minima for a certain configuration of obstacles and destinations, it is assumed to have a supervisory that provides a set of waypoints that describes a deadlock free trajectory to the destination. This trajectory can be found, for example, by using a SLAM (Simultaneous Localization And Mapping) system with an A$^*$ algorithm \cite{duchovn2014path}. Furthermore, the distance measured with this sensor is adjusted to consider the robot shape and kinematics as presented in \cite{minguez2006abstracting}.

The system model abstracts the robot as a punctual and omnidirectional vehicle because we assume that the measured distance to an obstacle is adjusted to the robot shape and kinematics. Thus, the differential equations that model the robot trajectory are: $x^{\prime}=v \cdot \cos \theta, y^{\prime} = v \cdot \cos \theta, v^{\prime}=a, \theta^{\prime}=\omega, \omega^{\prime}=\frac{a}{r_c}$, where $x$ and $y$ is the  the two dimensional position of the robot, $\theta$ is the robot heading angle, $r_c$ is the circular trajectory radius such as $r_c = \frac{v_0}{\omega^*}$ and $v_0$ is initial translational velocity. Now, we can formulate our problem as follows.

\begin{problem}
Given the maximum acceleration $A \geq 0$, deceleration $b > 0$, angular velocity $\Omega \geq 0$ and controller cycle time $\epsilon > 0$, find a trajectory defined by a pair of translational and angular velocity $\langle v^*, \omega^* \rangle$ that optimizes the robot translational velocity towards a goal waypoint and ensures the passive safety.
\end{problem}

\subsection{Verification}
In \cite{mitsch2013provably} and \cite{mitsch2016formal} the authors verified a hybrid control system modeled as a d$\mathcal{L}$ hybrid program in Model 2 \cite{mitsch2016formal}, which abstracts a motion planning that ensures the passive safety. The transition system that represents this control system is shown in Fig.~\ref{fig:controlenvelope}, where the $Drive$ mode allows the regulator to take any transnational acceleration $a$ in its domain (i.e. $-b \leq a \leq A$) and the robot must brake (i.e. $a = -b$) in the $Brake$ mode. The initial constraint $\phi_{ps}$ is that the robot start at rest (i.e. $\phi_{ps} \equiv v = 0$). The safe constraint $safe$ that ensures the passive safety $\psi_{ps}$ after any finite nondeterministic executions of the control system.
\begin{align}
	\psi_{ps} \equiv & v \neq 0 \rightarrow d > 0 \\
    safe \equiv & \frac{d}{\sqrt{2}} > \frac{v^2}{2b} + V\frac{v}{b} + \Big(\frac{A}{b}+1\Big)\Big(\frac{A}{2}\epsilon^2 + \epsilon(v+V)\Big)\label{eq:safe},
\end{align} where $d$ is the euclidean distance to the closest obstacle point. 

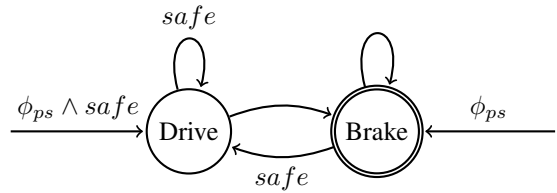
\begin{figure}[t]
    \centering
    \begin{tikzpicture}[shorten >=1pt,node distance=2.5cm,on grid,auto, bend angle=20, thick,scale=1, every node/.style={transform shape}]
    \node (s_0_1) {};
    \node[state] (s_2) [right=of s_0_1] {Drive};
    \node[state,accepting] (s_3) [right=of s_2] {Brake};
    \node (s_0_2) [right=of s_3] {};
    \path[->]
    (s_0_1) edge node [pos=0.5, sloped, above]{$\phi_{ps} \wedge safe$} (s_2)
    (s_0_2) edge node [pos=0.5, sloped, above]{$\phi_{ps}$} (s_3)
    (s_2) edge [loop above] node [pos=0.5, sloped, above]{$safe$} (s_2)
    (s_3) edge [loop above] node [pos=0.5, sloped, above]{} (s_3)
    (s_2) edge [bend left] node [pos=0.5, above]{} (s_3)
    (s_3) edge [bend left] node [pos=0.5, below]{$safe$} (s_2);
    \end{tikzpicture}
    \caption{Transition system that represents the hybrid control system verified in \cite{mitsch2013provably,mitsch2016formal}.}
    \label{fig:controlenvelope}
\end{figure}

\section{Certified Reactive Potential Field}\label{sec:approach}
The main contribution of the current work is to provide a reactive potential field approach SafeGuardPF for mobile robots to avoid moving obstacles, which is formally verified to be passively safe. Formal guarantee is obtained ensuring that the robot controls are constrained by the hybrid control system presented in the Fig.~\ref{fig:controlenvelope}. This control system was proved to guarantee the passive safety property in \cite{mitsch2013provably,mitsch2016formal}. The SafeGuardPF approach itself consists of two independently calculated components: Translational and Rotational Behaviors (see Fig.~\ref{fig:Behaviors}, where $v_i$ represent translational velocities and $\omega_i$ represent rotational speed). Intuitively speaking, this logic is similar to the logic of a car driver: increase/decrease speed and rotation of the wheel.

\begin{figure}[t]
    \centering
    \includegraphics[width=0.9\linewidth]{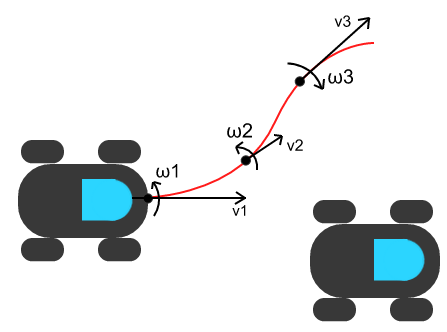}
    \caption{Translational and Rotational Behaviors.}
    \label{fig:Behaviors}
    
\end{figure}

Translational Behavior is responsible for maximizing translational velocity while keeping it not beyond the value, for which the equation (\ref{eq:safe}) holds. This equation gives an upper bound for safe distance to a moving obstacle with a maximum velocity equal to $V$. If we solve this equation for the translational velocity, we find the maximum safe translational velocity $v_{max}$.

\begin{prop}\label{prop:vmax}
The maximum velocity that the robot can assign any value to acceleration $a$ such that $-b \leq a \leq A$ is,
\begin{equation}\label{eq:vsafe}
    v < v_{max} = b \sqrt{\Big(\frac{A}{b} + 1\Big) \epsilon^2 + \frac{V^2}{b^2} + \sqrt{2}\frac{d}{b}} - V - \Big(\frac{A}{b} + 1\Big) \epsilon b
\end{equation}
\end{prop}

\begin{proof}
First, the Equation (\ref{eq:safe}) is rearranged as follows,
\begin{align*}
& \frac{d}{\sqrt{2}} > \frac{v^2}{2b} + V\frac{v}{b} + \Big(\frac{A}{b}+1\Big)\Big(\frac{A}{2}\epsilon^2 + \epsilon(v+V)\Big) \\
\Rightarrow & \frac{v^2}{2b} + \Big(\frac{V}{b} + \frac{A}{b}\epsilon+\epsilon\Big)v + \Big(\frac{A}{b}+1\Big)\Big(\frac{A}{2}\epsilon^2 + \epsilon V\Big) - \frac{d}{\sqrt{2}} \\
& < 0.
\end{align*}
Let $\alpha$, $\beta$ and $\gamma$ be the coefficients of a quadratic equation, thus,
\begin{align*}
& \alpha = \frac{1}{2b} \\
& \beta = \frac{V}{b} + \Big(\frac{A}{b} + 1\Big) \epsilon \\
& \gamma = \Big(\frac{A}{b}+1\Big)\Big(\frac{A}{2}\epsilon^2 + \epsilon V\Big) - \frac{d}{\sqrt{2}} \\
\Rightarrow & \alpha v_l^2 + \beta v_l + \gamma < 0
\end{align*}
Since $\alpha > 0$, $\beta > 0$ and $v \geq 0$, $\frac{-\beta - \sqrt{\Delta}}{2 \alpha} < 0$ and, 
\begin{equation*}
0 \leq v < \frac{-\beta + \sqrt{\Delta}}{2 \alpha},
\end{equation*} such that,
\begin{align*}
\Delta = & \beta^2 - 4 \alpha \gamma \\
= & \Big(\frac{V}{b} + \frac{A}{b}\epsilon+\epsilon\Big)^2 - 4 \frac{1}{2b} \Big[\Big(\frac{A}{b}+1\Big)\Big(\frac{A}{2}\epsilon^2 + \epsilon V\Big) - \frac{d}{\sqrt{2}}\Big]  \\
= & \frac{V^2}{b^2} + 2\Big(\frac{A}{b} + 1\Big)\epsilon\frac{V}{b} + \Big(\frac{A^2}{b^2}+2\frac{A}{b}+1\Big)\epsilon^2  \\
& - \Big(\frac{A^2}{b^2} + \frac{A}{b}\Big) \epsilon^2 -2 \Big(\frac{A}{b}+1\Big) \epsilon\frac{V}{b}+ \sqrt{2}\frac{d}{b} \\
= & \Big(\frac{A}{b} + 1\Big) \epsilon^2 + \frac{V^2}{b^2} + \sqrt{2}\frac{d}{b} 
\end{align*}

Therefore,
\begin{equation*}
0 \leq v < b \sqrt{\Big(\frac{A}{b} + 1\Big) \epsilon^2 + \frac{V^2}{b^2} + \sqrt{2}\frac{d}{b}} - V - \Big(\frac{A}{b} + 1\Big) \epsilon b.
\end{equation*}
\end{proof}

This equation is a function of the robot's upper bound on speed and relative position to the closest obstacle. Intuitively speaking, translational speed should be not greater than the bound for being able to come to a complete stop before the collision in the worst case scenario (if the closest obstacle would be approaching towards the controlled robot with a speed \textit{V}). 

\begin{prop}\label{prop:vstar}
The desired translational velocity $v^*$ which assigns the maximum safe desired velocity is,
\begin{equation}
v^* = \begin{cases}
v_{max} - \delta + A\epsilon, & \text{if } safe \text{ holds true}, \\
0, & \text{otherwise},
\end{cases}
\end{equation}where $\delta$ is a precision constant.
\end{prop}
\begin{proof}
    From Model 2 \cite{mitsch2016formal}, if a translational velocity $v$ satisfies the condition $safe$ for a given position and parameters, then the acceleration can be any value between $-b$ and $A$. Thus, if $safe$ holds true, the robot is allowed to accelerate up to $A$, and the desired velocity is the maximum safe velocity for the next control cycle which is $v_{max} - \delta + A \epsilon$. The precision constant $\delta$ is added to address the strict bound on the safe velocity. Otherwise, the robot must brake, and the desired velocity is zero.
\end{proof}

Meanwhile, Rotational Behavior is responsible for heading towards the destination, while avoiding obstacles. The force towards the destination is called attraction force and its resulting field is the attraction field $\vec{\mathcal{F}}_{att}$, which is calculated based on the angle $\beta$ towards the goal and attraction coefficient $K_{att}$.
\begin{align}
\vec{\mathcal{F}}_{att} = & K_{att}\frac{\vec{{u}}_{rob}-\vec{{u}}_{goal}}{\parallel \vec{{u}}_{rob}-\vec{{u}}_{goal} \parallel} = K_{att} \cdot \vec{u}_{att} \\
\vec{u}_{att} = & \langle \cos \beta, \sin \beta \rangle \\
\beta = & \angle (\vec{{u}}_{rob}-\vec{{u}}_{goal}),
\end{align}where $\vec{{u}}_{rob}$ and $\vec{{u}}_{goal}$ are vectors representing the robot and destination position, respectively. This force is constant, meaning that the robot will turn towards the object with the same velocity independent of how far it is to the destination. 

The maximum translational velocity is proportional to $v_{max}$, and this velocity is an increasing function of distance to the closest obstacle point. Thus, it is possible to maximize the translational velocity if the desired angular velocity $\omega^*$ points towards the gradient of the maximum safe velocity, i.e. $\bigtriangledown v_{max}$. Hence, it is defined a force that maximizes the translational velocity and, consequently, avoids the obstacles which is called repulsion field $\vec{\mathcal{F}}_{rep}$. This field is calculated based on the distance $d$ and angle $\alpha$ to the closest obstacle and the parameter repulsion coefficient $K_{rep}$. The Fig.~\ref{fig:grad_square} illustrates this repulsion field for a square obstacle and a mass point robot.

\begin{prop}
The potential field which drives the robot to the maximum safe velocity $v_{max}$ is
\begin{align}
\vec{\mathcal{F}}_{rep} = & K_{rep} \frac{\partial v_{max}}{\partial d} \vec{u}_{obs}, \\
\vec{u}_{obs} = & \langle \cos \alpha, \sin \alpha \rangle.
\end{align}
Such that,
\begin{align}
\frac{\partial v_{max}}{\partial d} = & \frac{1}{\sqrt{\Big(\frac{A}{b} + 1\Big) \epsilon^2 + \frac{V^2}{b^2} + \sqrt{2}\frac{d}{b}}},
\end{align}where $x_o$ and $y_o$ are the 2-dimensional position of the closest obstacle point and maximum gradient $\bigtriangledown_{v_{max}}^{(\max)}$ is a parameter that defines the maximum value that the gradient $\bigtriangledown_{v_{max}}$ can take. 
\end{prop}
 
\begin{proof}
The repulsion field is proportional to the gradient of maximum safe distance $v_{max}$ such that,
\begin{align*}
\vec{\mathcal{F}}_{rep} = & K_{rep} \bigtriangledown v_{max} \\
= & \langle \frac{\partial v_{max}}{\partial x},  \frac{\partial v_{max}}{\partial y}\rangle \\
= & \frac{1}{\sqrt{\Big(\frac{A}{b} + 1\Big) \epsilon^2 + \frac{V^2}{b^2} + \sqrt{2}\frac{d}{b}}} \langle \frac{x - x_o}{d}, \frac{y - y_o}{d} \rangle
\end{align*}
However,
\begin{align*}
\cos \alpha = & \frac{x - x_o}{d} \\
\sin \alpha = & \frac{y - y_o}{d}
\end{align*}
Since $\vec{u}_{obs} = \langle \cos \alpha, \sin \alpha \rangle$,
\begin{align*}
\vec{\mathcal{F}}_{rep} = & K_{rep} \frac{\partial v_{max}}{\partial d} \vec{u}_{obs}
\end{align*}
\end{proof}

\begin{figure}[t]
    \centering
    \includegraphics[width=0.9\linewidth]{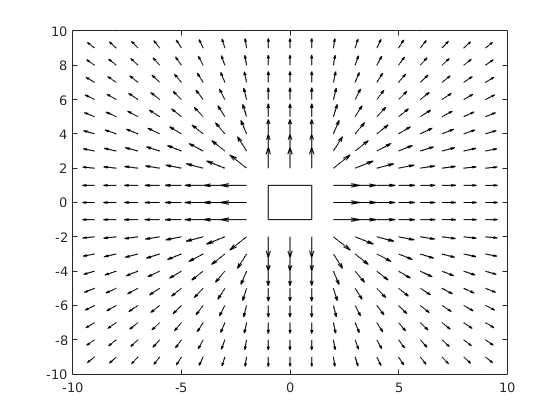}
    \caption{Repulsion field for a squared moving obstacle (V=0.75m/s). This force was sampled for positions in the range $\langle [-10,10], [-10,10] \rangle$ considering the following parameters: $v=1m/s$, $A=b=0.3m/s^2$, $\epsilon=0.1s$, $K_{rep}=1$ and $\bigtriangledown_{v_{max}}^{(\max)} = 1$.}
    \label{fig:grad_square}
\end{figure}

Finally, the sum of those two fields will define the desired heading that is used to calculate the desired angular velocity. 
\begin{align}
\vec{\mathcal{F}} = & \vec{\mathcal{F}}_{att} + \vec{\mathcal{F}}_{rep}
\end{align}
The Fig.~\ref{fig:potfield} shows a potential $\vec{\mathcal{F}}$ for a moving obstacle. 

\begin{figure}[t]
    \centering
    \includegraphics[width=0.9\linewidth]{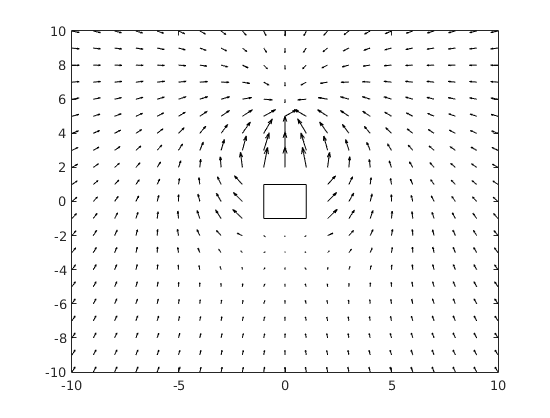}
    \caption{Potential field for a squared moving obstacle (V=0.75m/s) for an attraction field towards the position $\langle 0,5 \rangle$. This force was sampled for positions in the range $\langle [-10,10], [-10,10] \rangle$ considering the following parameters: $v=1m/s$, $A=b=0.3m/s^2$, $\epsilon=0.1s$, $K_{rep}=1$ and $K_{att}=0.25$.}
    \label{fig:potfield}
\end{figure}

Finally, the desired angular velocity $\omega^*$ is calculated based on the actual robot heading $\theta$ and the potential field $\vec{\mathcal{F}}$ such that,
\begin{equation}
\omega^* = \begin{cases}
|\vec{\mathcal{F}}| \big(\theta - \angle \vec{\mathcal{F}}\big) & -\Omega \leq \omega^* \leq \Omega \\
-\Omega & \omega^* < -\Omega \\
\Omega & \omega^* > \Omega.
\end{cases}
\end{equation}

In summary, the SafeGuardPF approach finds a pair of desired translational and angular velocities $\langle v^*, \omega^* \rangle$ which depends only on current robot states, i.e. speed and relative distance to the closest obstacle point. Furthermore, the desired translational velocity $v^*$ ensures that the robot will not collide to any obstacle moving up to velocity $V$, or it will be at rest. Therefore, this algorithm is efficient, scalable and proven safe.

\section{Simulation and Experimental Results}\label{sec:results}

The mobile robots used in simulation and experiments are Pioneer P3-AT robots\footnote{http://www.mobilerobots.com/ResearchRobots/P3AT.aspx, retrieved 09-12-2016.}, which has a development kit called the Pioneer SDK\footnote{http://www.mobilerobots.com/Software.aspx, retrieved 09-12-2016.}. The SafeGuardPF was implemented in a custom C++ application using the libraries ARIA\footnote{http://www.mobilerobots.com/Software/ARIA.aspx, retrieved 09-12-2016.} and ARNL\footnote{http://www.mobilerobots.com/Software/NavigationSoftware.aspx, retrieved 05-18-2016.}. The ARIA library brings an interface to control and to receive data from MobileSim\footnote{http://www.mobilerobots.com/Software/MobileSim.aspx, retrieved 09-12-2016.} accessible via a TCP port and is the foundation for all other software libraries in the SDK such as the ARNL. The MobileSim permits to simulate all current and legacy models of MobileRobots/ActivMedia mobile robots including the Pioneer 3 AT. The ARNL Navigation library\footnote{http://www.mobilerobots.com/Software/NavigationSoftware.aspx, retrieved 09-12-2016.} provides a MobileRobots' proprietary navigation technology that is reliable, high quality and highly configurable and implements an intelligent localization capabilities to the robot. Different localization methods are available for different sensors such as LIDAR, Sonar and GPS. Finally, a Pioneer SDK implementation can be controlled by the MobileEye\footnote{http://www.mobilerobots.com/Software/MobileEyes.aspx, retrieved 09-12-2016.}, which is a graphical interface that can send commands and read data from ARIA and ARNL to show the sensor readings and trajectories. Therefore, each controller is implemented in a C++ custom application that is connected via a TCP/IP port to a MobileEye application and runs on Linux Computers. The simulation testbed consists of a MobileSim which simulates each robot dynamics and also runs on Linux Computers, where each robot controller connects via a TCP port. In experimental results, the MobileSim is substituted by an actual robot. The Fig.~\ref{fig:bbr} illustrate graphically the SafeGuardPF implementation used in the simulations and experiments. 

\begin{figure}[t]
    \centering
    \includegraphics[width=0.9\linewidth]{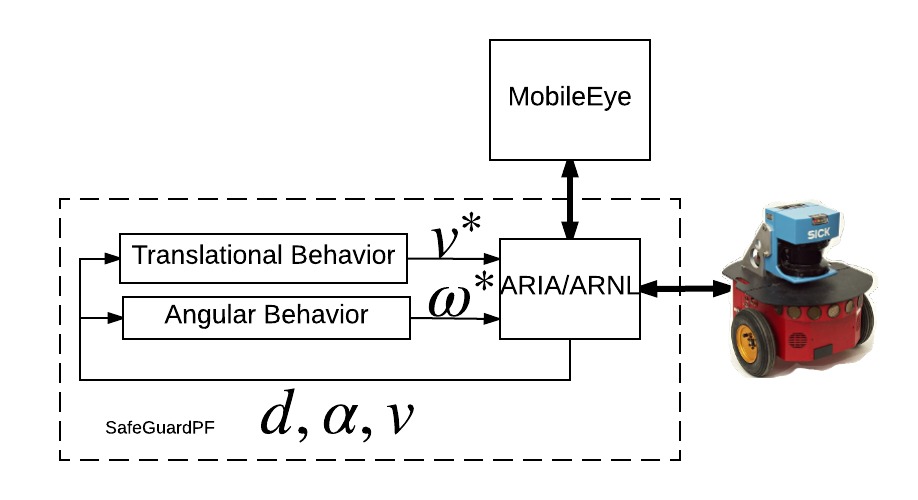}
    \caption{The SafeGuardPF implementation as Behavior-Based Robotics approach using Pioneer SDK development kit. The ARIA and ARNL libraries coordinate the behaviors and connect them to the robot. The MobileEye provides a graphical interface that allows monitoring and sending commands to the robot controller.}
    \label{fig:bbr}
\end{figure}

The simulated and real robots were equipped with laser sensor SICK LMS-500. This sensor was used to detect obstacles and for indoor localization. The Pioneer P3-AT robots shape can be approximated to a circle; thus, the distance to the closest obstacle point is the minimum measured distance minus the robot radius. 

\begin{example}[Simulation with 2 Pioneer robots]\label{ex:simrobot}
In Fig.~\ref{fig:Simulation} are shown two Pioneer P3-AT robots (R1 and R2) with lasers, which are heading towards their destinations, while avoiding obstacles (as a demonstration of the implemented SafeGuardPF approach). The origin of the arrow shows the robot's location, the length of the arrow - value of robot's speed and arrow's direction - the direction of the velocity vector. The shape and dimensions (roughly 4.5 by 4.5 meters) of the map duplicate the shape and dimensions of the arena for experimental implementation on real robots. The potential fields coefficients used were $K_{att} = 0.04$ and $K_{rep} = 0.09$.

In Fig.~\ref{fig:Distance} on the top is shown the distance to the closest obstacle (from a laser sensor) and on the bottom are shown translational and angular velocities. In this scenario robot 1 had initial position close to a wall (Fig.~\ref{fig:Simulation}) and its speed was limited because of the presence of a close obstacle (but, in fact, this assumption can be improved, as we treat all obstacles as obstacles which may have a maximum speed \textit{V}, but in reality walls are usually not moving and as a future work we plan to distinguish movable and non-movable obstacles to be less conservative). As the robot moves and become farther from the wall, it is able to continuously increase its speed until it notices robot 2. As robot 1 approached to robot 2, it had to decrease its speed to prevent the collision by coming to a complete stop even in the worst case scenario (if the second robot will be approaching robot 1 with velocity \textit{V}). At the same time robot 1 was looking for a better direction to go and decided to turn left. Because this direction was a good choice, the robot 1 was moving away from the robot 2 and as the distance to the closest obstacle was increasing, the robot 1 could increase its velocity and continue its movement towards the destination.
\end{example}

\begin{example}[Implementation on a real robot]\label{ex:realrobot}
During a demonstration of collision avoidance on a real Pioneer P3-AT robot, the robot had to move towards its destination, while avoiding an unpredictable moving obstacle (represented by a human), as shown in Fig.~\ref{fig:Position_real}. The coefficients were set as $K_{att} = 0.04$ and $K_{rep} = 0.12$ because in the real experiment the rotational velocity realization was less precise than during the simulation.

In this scenario a human was approaching the robot on the right side, almost perpendicular to its movement. As the robot noticed the human, it had to decrease its speed (we can see it as a smaller length of the velocity vector in coordinates around 3000 in X and 1500 in Y in Fig.~\ref{fig:Position_real}). The robot obeyed the rule that it had to always have enough distance to be able to come to a complete stop, if the human would be approaching the robot directly with velocity \textit{V}. As the human was passing by, at first the robot decided to turn left (see Fig.~\ref{fig:Distance_real}, around 2 sec), but as the human completely passed by and he was noticed by a laser sensor to be on the left of the robot, the robot decided to turn right. After that the robot continued its path towards its destination. As it was approaching the destination, it noticed that the walls were close and started decreasing its speed to be safe for sure. Then the human was noticed again in the direction of the destination and robot had to greatly decelerate again (see Fig.~\ref{fig:Distance_real}, around 6 sec) and finish its path when the human left the destination.
\end{example}

The experiments (in simulation and on real robots) were also done with sonar sensors only, and robots were always able to successfully prevent collisions, but localization precision was less accurate (it affected how close robots approached their destinations).

\begin{figure}[t]
    \centering
    \includegraphics[width=0.9\linewidth]{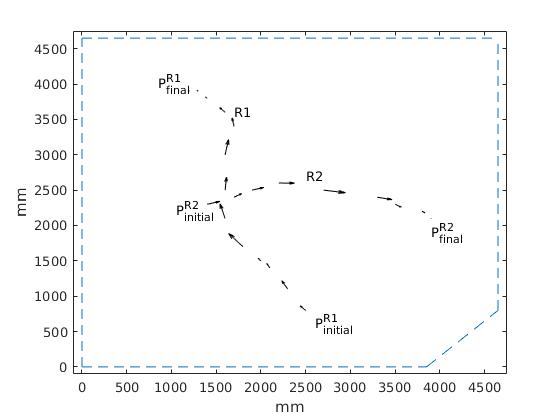}
    \caption{Simulation of collision avoidance for 2 Pioneer robots with lasers.}
    \label{fig:Simulation}
\end{figure}

\begin{figure}[t]
    \centering
    \includegraphics[width=0.9\linewidth]{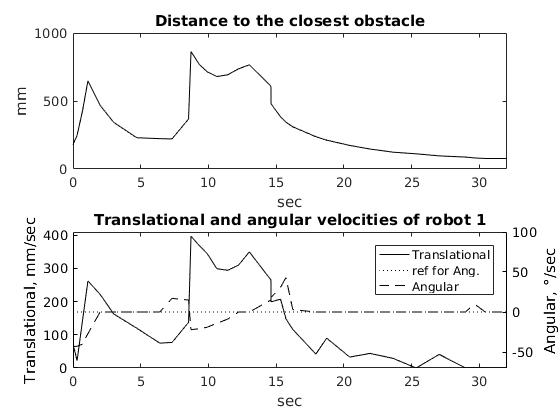}
    \caption{Distance to the closest obstacle; translational and angular velocities of robot 1.}
    \label{fig:Distance}
\end{figure}

\begin{figure}[t]
    \centering
    \includegraphics[width=0.9\linewidth]{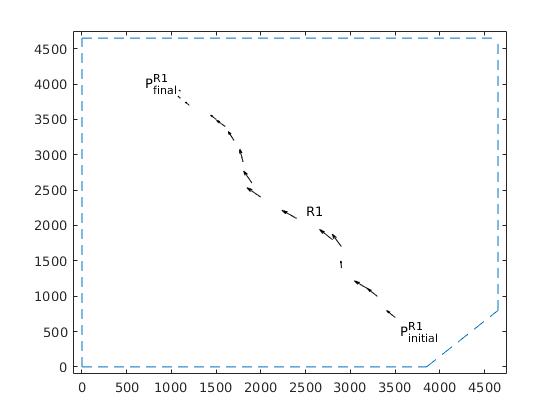}
    \caption{Demonstration of collision avoidance for a Pioneer with a laser and an unpredictable obstacle (human).}
    \label{fig:Position_real}
\end{figure}

\begin{figure}[t]
    \centering
    \includegraphics[width=0.9\linewidth]{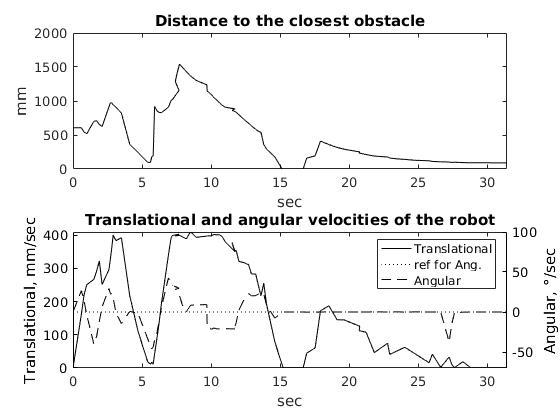}
    \caption{Distance to the closest obstacle; translational and angular velocities of the robot 1 during implementation on a real robot with an unpredictable obstacle (a human).}
    \label{fig:Distance_real}
\end{figure}

\subsection{Discussion}\label{sec:discussion}

As a practical result from the Section \ref{sec:results}, we may infer that the proposed approach is provable safe and practically implementable and hence might be used in safety critical environments (for example, in scenarios with human workers in a warehouse or a manufactory).  The safety is ensured because the desired translational velocity $v^*$ is bounded using the safe constraint presented in Equation (27) in Model 2 \cite{mitsch2016formal} which guarantees a passive safety property (see Propositions \ref{prop:vmax} and \ref{prop:vstar}). Hence, the robot will always keep a speed that allows it to stop before colliding with a moving obstacle. Furthermore, it was not assumed to know intentions of the moving obstacle, but only its maximum velocity $V$. Thus, it can prevent collisions with obstacles having unpredictable intentions (e.g. human workers), if the maximum speed of the obstacles does not exceed $V$. Example \ref{ex:realrobot} shows that a real robot crossing with a human could ensure a safe velocity. 


The proposed method is efficient because it is a reactive control system which safe trajectories were found using an offline verification and the desired trajectory $\langle v^*, \omega^* \rangle$ does not require a numerical optimization. Furthermore, only the knowledge about the closest obstacle point is required to ensure the safety; thus, it is scalable because the number of agents in the environment does not affect the computation cost. For instance, Example \ref{ex:simrobot} shows two robots with decentralized control systems (identical to each other) whose safe trajectories were found without knowledge of the other robot's future kinematics. Therefore, at each computation cycle, the control system is able to find short term safe trajectories for a current state using only the current relative distance and its own velocity. This approach is less computationally expensive that an online estimation of the obstacle future kinematics.

\section{Conclusion}\label{sec:conclusion}
This paper presents an efficient and proven safe reactive potential field approach. The algorithm ensures the proved conditions presented in \cite{mitsch2013provably,mitsch2016formal} for ground mobile robots and uses it to find a potential field that leads the robot to maximized translational velocity towards to the destination. Since this potential field is found using only current robot states (i.e. translational velocity $v$, relative distance $d$ and angle $\alpha$ to the closest obstacle point), it does not require costly computation. Furthermore, the safety is ensured for moving obstacles with maximum translational velocity $V$ without requiring an online model predictive algorithm. Therefore, the SafeGuardPF approach brings a complex predictive behavior using simple and efficient behaviors to ground mobile robots.

Moreover, there were presented results for simulated and real robots. Those experiments show the emerging behavior even in complex scenarios such as a robot sharing its workspace with a human. Since these behaviors require only the closest obstacle point to find the desired velocities (i.e. $\langle v^*, \omega^* \rangle$), the algorithm is scalable for multi-agent scenarios.

\small
\bibliographystyle{plain}

\bibliography{SafeGuardPF}

\begin{thebibliography}{10}

\bibitem{althoff2012safety}
Daniel Althoff, James~J Kuffner, Dirk Wollherr, and Martin Buss.
\newblock Safety assessment of robot trajectories for navigation in uncertain
  and dynamic environments.
\newblock {\em Autonomous Robots}, 32(3):285--302, 2012.

\bibitem{althoff2014online}
Matthias Althoff and John~M Dolan.
\newblock Online verification of automated road vehicles using reachability
  analysis.
\newblock {\em IEEE Transactions on Robotics}, 30(4):903--918, 2014.

\bibitem{bouraine2012provably}
Sara Bouraine, Thierry Fraichard, and Hassen Salhi.
\newblock Provably safe navigation for mobile robots with limited
  field-of-views in dynamic environments.
\newblock {\em Autonomous Robots}, 32(3):267--283, 2012.

\bibitem{braunl2008embedded}
Thomas Br{\"a}unl.
\newblock {\em Embedded robotics: mobile robot design and applications with
  embedded systems}.
\newblock Springer Science \& Business Media, 2008.

\bibitem{budhraja2016approach}
Akshit Budhraja, Roopak Srivastava, and PM~Pradhan.
\newblock An approach to adaptive swarm surveillance using social potential
  fields.
\newblock In {\em 2016 IEEE 6th International Conference on Advanced Computing
  (IACC)}, pages 191--196. IEEE, 2016.

\bibitem{buniyamin2011simple}
N~Buniyamin, WAJ Wan~Ngah, N~Sariff, and Z~Mohamad.
\newblock A simple local path planning algorithm for autonomous mobile robots.
\newblock {\em International journal of systems applications, Engineering \&
  development}, 5(2):151--159, 2011.

\bibitem{duchovn2014path}
Franti{\v{s}}ek Ducho{\v{n}}, Andrej Babinec, Martin Kajan, Peter Be{\v{n}}o,
  Martin Florek, Tom{\'a}{\v{s}} Fico, and Ladislav Juri{\v{s}}ica.
\newblock Path planning with modified a star algorithm for a mobile robot.
\newblock {\em Procedia Engineering}, 96:59--69, 2014.

\bibitem{fraichard2004inevitable}
Thierry Fraichard and Hajime Asama.
\newblock Inevitable collision states—a step towards safer robots?
\newblock {\em Advanced Robotics}, 18(10):1001--1024, 2004.

\bibitem{ge2002dynamic}
Shuzhi~S. Ge and Yun~J Cui.
\newblock Dynamic motion planning for mobile robots using potential field
  method.
\newblock {\em Autonomous Robots}, 13(3):207--222, 2002.

\bibitem{ge2000new}
Shuzhi~Sam Ge and Yan~Juan Cui.
\newblock New potential functions for mobile robot path planning.
\newblock {\em IEEE Transactions on robotics and automation}, 16(5):615--620,
  2000.

\bibitem{hamani2012mobile}
M~Hamani and A~Hassam.
\newblock Mobile robot navigation in unknown environment using improved apf
  method.
\newblock In {\em the 13th international conference in unknown environment
  using improved APF method}, pages 0875--1812, 2012.

\bibitem{hess2014formal}
Daniel He{\ss}, Matthias Althoff, and Thomas Sattel.
\newblock Formal verification of maneuver automata for parameterized motion
  primitives.
\newblock In {\em 2014 IEEE/RSJ International Conference on Intelligent Robots
  and Systems}, pages 1474--1481. IEEE, 2014.

\bibitem{khatib1986real}
Oussama Khatib.
\newblock Real-time obstacle avoidance for manipulators and mobile robots.
\newblock {\em The international journal of robotics research}, 5(1):90--98,
  1986.

\bibitem{lin2014mission}
Hai Lin.
\newblock Mission accomplished: An introduction to formal methods in mobile
  robot motion planning and control.
\newblock {\em Unmanned Systems}, 2(02):201--216, 2014.

\bibitem{macek2009towards}
Kristijan Macek, Dizan Alejandro~Vasquez Govea, Thierry Fraichard, and Roland
  Siegwart.
\newblock Towards safe vehicle navigation in dynamic urban scenarios.
\newblock {\em Automatika}, 2009.

\bibitem{matoui2015local}
Fethi Matoui, Boumedyen Boussaid, and Mohamed~Naceur Abdelkrim.
\newblock Local minimum solution for the potential field method in multiple
  robot motion planning task.
\newblock In {\em 2015 16th International Conference on Sciences and Techniques
  of Automatic Control and Computer Engineering (STA)}, pages 452--457. IEEE,
  2015.

\bibitem{minguez2006abstracting}
Javier Minguez, Luis Montano, and Jos{\'e} Santos-Victor.
\newblock Abstracting vehicle shape and kinematic constraints from obstacle
  avoidance methods.
\newblock {\em Autonomous Robots}, 20(1):43--59, 2006.

\bibitem{mitsch2013provably}
Stefan Mitsch, Khalil Ghorbal, and Andr{\'e} Platzer.
\newblock On provably safe obstacle avoidance for autonomous robotic ground
  vehicles.
\newblock In {\em Robotics: Science and Systems}, 2013.

\bibitem{mitsch2016formal}
Stefan Mitsch, Khalil Ghorbal, David Vogelbacher, and Andr{\'e} Platzer.
\newblock Formal verification of obstacle avoidance and navigation of ground
  robots.
\newblock {\em arXiv preprint arXiv:1605.00604}, 2016.

\bibitem{platzer2010logical}
Andr{\'e} Platzer.
\newblock {\em Logical analysis of hybrid systems: proving theorems for complex
  dynamics}.
\newblock Springer Science \& Business Media, 2010.

\bibitem{raja2012optimal}
P~Raja and S~Pugazhenthi.
\newblock Optimal path planning of mobile robots: A review.
\newblock {\em International Journal of Physical Sciences}, 7(9):1314--1320,
  2012.

\bibitem{savkin2015safe}
Andrey~V Savkin, Alexey~S Matveev, Michael Hoy, and Chao Wang.
\newblock {\em Safe Robot Navigation Among Moving and Steady Obstacles}.
\newblock Butterworth-Heinemann, 2015.

\bibitem{vseda2007roadmap}
Milo{\v{s}} {\v{S}}eda.
\newblock Roadmap methods vs. cell decomposition in robot motion planning.
\newblock In {\em Proceedings of the 6th WSEAS International Conference on
  Signal Processing, Robotics and Automation}, pages 127--132. World Scientific
  and Engineering Academy and Society (WSEAS), 2007.

\bibitem{shiller2001motion}
Zvi Shiller, Frederic Large, and Sepanta Sekhavat.
\newblock Motion planning in dynamic environments: Obstacles moving along
  arbitrary trajectories.
\newblock In {\em Robotics and Automation, 2001. Proceedings 2001 ICRA. IEEE
  International Conference on}, volume~4, pages 3716--3721. IEEE, 2001.

\bibitem{sun2016flight}
Fanrong Sun and Songchen Han.
\newblock A flight path planning method based on improved artificial potential
  field.
\newblock In {\em 2016 International Conference on Computer, Information and
  Telecommunication Systems (CITS)}, pages 1--5. IEEE, 2016.

\bibitem{thrun1997dynamic}
D~Fox W Burgard~S Thrun, D~Fox, and W~Burgard.
\newblock The dynamic window approach to collision avoidance.
\newblock {\em IEEE Transactions on Robotics and Automation}, 4:1, 1997.

\bibitem{vadakkepat2000evolutionary}
Prahlad Vadakkepat, Kay~Chen Tan, and Wang Ming-Liang.
\newblock Evolutionary artificial potential fields and their application in
  real time robot path planning.
\newblock In {\em Evolutionary Computation, 2000. Proceedings of the 2000
  Congress on}, volume~1, pages 256--263. IEEE, 2000.

\bibitem{yin2009new}
Lu~Yin, Yixin Yin, and Cheng-Jian Lin.
\newblock A new potential field method for mobile robot path planning in the
  dynamic environments.
\newblock {\em Asian Journal of Control}, 11(2):214--225, 2009.

\end{thebibliography}

\end{document}